\documentclass{article}
\pdfoutput=1

\usepackage{amsmath,amsfonts,bm}

\def\1{\bm{1}}

\def\rvu{{\mathbf{i}}}

\def\rvn{{\mathbf{n}}}

\def\rvt{{\mathbf{t}}}
\def\rvu{{\mathbf{u}}}
\def\rvv{{\mathbf{v}}}

\def\rvx{{\mathbf{x}}}

\def\rmA{{\mathbf{A}}}

\def\rmC{{\mathbf{C}}}
\def\rmD{{\mathbf{D}}}

\def\rmH{{\mathbf{H}}}

\def\rmK{{\mathbf{K}}}
\def\rmL{{\mathbf{L}}}

\def\rmR{{\mathbf{R}}}
\def\rmS{{\mathbf{S}}}

\def\rmU{{\mathbf{U}}}
\def\rmV{{\mathbf{V}}}
\def\rmW{{\mathbf{W}}}
\def\rmX{{\mathbf{X}}}

\DeclareMathAlphabet{\mathsfit}{\encodingdefault}{\sfdefault}{m}{sl}
\SetMathAlphabet{\mathsfit}{bold}{\encodingdefault}{\sfdefault}{bx}{n}

\def\sR{{\mathbb{R}}}

\newcommand{\R}{\mathbb{R}}

\DeclareMathOperator*{\argmin}{arg\,min}

\newcommand{\intint}[1]{\left \llbracket#1\right \rrbracket}

\usepackage[table,xcdraw]{xcolor}
\usepackage[utf8]{inputenc} % allow utf-8 input
\usepackage[T1]{fontenc}    % use 8-bit T1 fonts
\usepackage{hyperref}       % hyperlinks
\usepackage{url}            % simple URL typesetting
\usepackage{booktabs}       % professional-quality tables
\usepackage{amsfonts}       % blackboard math symbols
\usepackage{nicefrac}       % compact symbols for 1/2, etc.
\usepackage{microtype}      % microtypography
\usepackage{stmaryrd}
\usepackage{todonotes}
\usepackage{dsfont}

\usepackage{amsthm}

\usepackage{algorithm}
\usepackage{algorithmic}
\usepackage{varwidth}
\usepackage{graphicx}
\usepackage{enumitem}
\usepackage{color}
\usepackage{xspace}
\usepackage{subcaption}
\usepackage{caption} 
\usepackage{array}
\newcolumntype{P}[1]{>{\centering\arraybackslash}p{#1}}
\newcolumntype{M}[1]{>{\centering\arraybackslash}m{#1}}

\newcommand{\addVE}[1]{{\color{blue}  #1}}

\def\datadim{D}
\def\nexamples{N}
\def\nclusters{K}
\def\nfactors{Q}

\newcommand{\bigO}[1]{\mathcal{O}\left (#1\right )}
\def\qkmeans{\texttt{QK-means}\xspace}
\def\kmeans{\texttt{K-means}\xspace}
\def\palm{\texttt{palm4MSA}\xspace}

\newcommand\norm[1]{\left \| #1\right \|}

\def\eqdef{:=}

\graphicspath{{./figures/}}

\usepackage{authblk}
\usepackage{fullpage}
\newtheorem*{remark}{Remark}
\newtheorem*{proposition}{Proposition}

\title{QuicK-means: Acceleration of K-means by learning a fast transform}

\author[1]{Luc Giffon}
\author[1]{Valentin Emiya}
\author[2,1]{Liva Ralaivola}
\author[1]{Hachem Kadri}
\affil[1]{Aix Marseille Univ, CNRS, LIS, Marseille, France}
\affil[2]{Criteo}

\begin{document}

\maketitle

\begin{abstract}
	
	\kmeans -- and the celebrated Lloyd algorithm -- is more than the clustering method it was originally designed to be. 
	It has indeed proven pivotal to help increase the speed of many machine learning and data analysis techniques such as indexing, nearest-neighbor search and prediction, data compression; its beneficial use has been shown to carry over to the acceleration of kernel machines (when using the Nyström method). 
	Here, we propose a fast extension of \kmeans, dubbed \texttt{QuicK-means}, that rests on the idea of expressing the matrix of the $\nclusters$ centroids as a product of sparse matrices, a feat made possible by recent results devoted to find approximations of matrices as a product of sparse factors. Using such a decomposition squashes the complexity of the matrix-vector product between the factorized $\nclusters \times \datadim$ centroid matrix $\mathbf{U}$ and any vector from $\mathcal{O}(\nclusters \datadim)$ to $\mathcal{O}(A \log A+B)$, with $A=\min (\nclusters, \datadim)$ and $B=\max (\nclusters, \datadim)$, where $\datadim$ is the dimension of the training data. This drastic computational saving has a direct impact in the assignment process of a point to a cluster, meaning that it is not only tangible at prediction time, but also at training time, provided the factorization procedure is performed during Lloyd's algorithm. We precisely show that resorting to a factorization step at each iteration does not impair the convergence of the optimization scheme and that, depending on the context, it may entail a reduction of the training time. Finally, we provide discussions and numerical simulations that show the versatility of our computationally-efficient  \texttt{QuicK-means} algorithm. 
\end{abstract}

% OUR PAPER
%====================================================================================
\section{Introduction}

\kmeans is one of the most popular clustering algorithms~\cite{hartigan1979algorithm,jain2010data}. It can be used beyond clustering, for other tasks such as indexing, data compression,  nearest-neighbor search and prediction, and local network community detection~\cite{muja2014scalable,van2016local}. \kmeans is also a pivotal process to help increase the speed and the accuracy of many machine learning techniques such as the Nyström approximation of kernel machines~\cite{si2016computationally} and RBF networks~\cite{que2016back}.
The  conventional  \kmeans  algorithm  has  a  complexity  of~$\bigO{\nexamples \nclusters \datadim}$ per iteration, where $\nexamples$ is the number of data points, $\nclusters$ the number of clusters and $\datadim$ is the dimension of the data points.
However, the larger the number of clusters, the more iterations are needed to converge~\cite{arthur2006slow}.
As data dimensionality and data sample size continue to grow, it is critical to produce viable and cost-effective alternatives to the computationally expensive conventional \kmeans. 
Previous attempts to alleviate the computational issues in \kmeans often relied on batch-, sparsity- and randomization-based methods~\cite{Sculley2010Web, boutsidis2014randomized,shen2017compressed,liu2017sparse}.

Fast transforms have recently received increased attention in machine learning community as they can be used  to speed up random projections~\cite{le2013fastfood,gittens2016revisiting} and to improve landmark-based approximations~\cite{si2016computationally}.
These works primarily focused on fast transforms such as Fourier and Hadamard transforms, which are fixed before the learning begins. An interesting question is whether one can go beyond that and learn the fast transform from data. 
In a recent paper~\cite{LeMagoarou2016Flexible}, the authors introduced a sparse matrix approximation scheme aimed  at  reducing the  complexity  of  applying  linear  operators  in  high  dimension by   approximately   factorizing   the   corresponding   matrix   into few   sparse   factors. One interesting observation is that fast transforms, such as the  Hadamard  transform  and  the  Discrete  Cosine  transform, can be exactly or approximately decomposed as a product of sparse matrices.
In this paper, we take this idea further and investigate attractive and computationally less costly implementations of the \kmeans algorithm by learning a fast transform from data.
Specifically, we make the following contributions:
\begin{itemize}
	\item we introduce \texttt{QuicK-means}, a fast extension of \kmeans that rests on the idea of expressing the matrix of the $K$ centroids as a product of sparse matrices, a feat made possible by recent results devoted to find approximations of matrices as a product of sparse factors,
	\item we show that each update step in one iteration of our algorithm  reduces the overall objective, which is enough to guarantee the convergence of \texttt{QuicK-means},
	\item we perform a complexity analysis of our algorithm, showing that the computational gain in \texttt{QuicK-means}  has a direct impact in the assignment process of a point to a cluster, meaning that it is not only tangible at prediction time, but also at training time,
	\item we provide an empirical evaluation of \texttt{QuicK-means}  performance which demonstrates its effectiveness on different datasets in the contexts of clustering and kernel Nystr\"om approximation.
\end{itemize}

%!TEX root=neurips2019_qmeans.tex
\section{Preliminaries}
\label{sec:background}
We briefly review the basics of \kmeans and give background on learning fast transforms.
To  assist  the  reading,  we  list  the notations used in the paper in Table~\ref{tab:notation}.

%!TEX root=neurips2019_qmeans.tex
 
%\paragraph{Notations}

%%%%%%%%%%%%%%%%%%%%%%%%%%%%%%%%%%%%%%%%%%%%%%%%%%%%%%%%%%%%
\begin{table}[t]
	\centering
	\begin{footnotesize}
	\begin{tabular}{cl}\\
\hline
		{\bf Symbol}  & {\bf Meaning}\\
\hline
$\intint{M}$  & set of integers from $1$ to $M$\\
$\|\cdot\|$ & $L_2$-norm\\
$\|\cdot\|_F$ &    Frobenius norm  \\
$\|\cdot\|_0$ & $L_0$-norm\\
$\|\cdot\|_2$    &    spectral norm  \\
$\rmD_\rvv$ & diagonal matrix with vector $\rvv$ on the diagonal\\                                                          
\hline
$N$           & number of data points\\
$D$           & data dimension\\
$K$           & number of clusters\\
$Q$           & number of sparse factors\\
$\rvx_1,\ldots, \rvx_N $        &    data points\\
$\rmX \in\mathbb{R}^{N\times D}$&    data matrix\\
$\rvt$        &  cluster assignment vector\\
$\rvu_1,\ldots, \rvu_K $        &    \kmeans centroids\\
$\rmU\in\mathbb{R}^{K\times D}$ &    \kmeans centroid matrix\\
$\rvv_1,\ldots, \rvv_K $        &    \qkmeans centroids\\
$\rmV\in\mathbb{R}^{K\times D}$ &    \qkmeans centroid matrix\\
$\rmS_1, \ldots, \rmS_Q$        &    sparse matrices\\
$\mathcal{E}_1, \ldots, \mathcal{E}_Q$ & sparsity constraint sets\\
$\delta_{\mathcal{E}}$ & 		indicator functions for set $\mathcal{E}$\\
$\tau$  & current iteration \\
\hline
	\end{tabular}
	\end{footnotesize}
	\caption{Notation used in this paper.}
	\label{tab:notation}
\end{table}

\subsection{\kmeans}
\label{sec:kmeans}
The \kmeans algorithm is used to partition a set $\rmX=\{\rvx_1,\ldots,\rvx_N\}$ of $N$  vectors $\rvx_n\in\R^{\datadim}$ into a predefined number $\nclusters$ of clusters
with the aim of minimizing the distance between each $\rvx_n$ to the center $\rvu_k\in\R^{D}$
of the cluster $k$ it belongs to ---the center $\rvu_k$ of cluster $k$ is the
 mean vector of the points assigned to cluster $k$.
\kmeans attempts to solve
\begin{equation}
\label{eq:kmean_problem}
    \argmin_{\rmU, \rvt} \sum_{k\in\intint{\nclusters}} \sum_{n: t_n = k} \|\rvx_{n} -\rvu_{k}\|^2,
\end{equation}
where $\rmU=\{\rvu_1,\ldots,\rvu_K\}$ is the set of cluster centers and $\rvt \in  \intint{\nclusters}^{\nexamples}$ is the assignment vector that puts $\rvx_n$ in cluster $k$
if $t_n=k$.

\paragraph{Lloyd's algorithm.} The most popular procedure to (approximately) 
solve the \kmeans problem is the iterative Lloyds algorithm, which alternates
i) an assignment step that decides the current cluster to which each point $\rvx_n$
belongs and ii) a reestimation step which refines the clusters and their centers.
In little more detail, the algorithm starts with an initialized set of $\nclusters$
 cluster centers $\rmU^{(0)}$ and proceeds as follows: at iteration $\tau$,
  the assignments are updated as
\begin{align}
\label{eq:assignment_problem_kmeans}
\forall n\in\intint{N}, t_n^{(\tau)} \leftarrow \argmin_{k \in \intint{\nclusters]}} \left\|\rvx_{n} - \rvu_{k}^{(\tau-1)}\right\|_2^2 = \argmin_{k \in \intint{\nclusters}} \left\|\rvu_{k}^{(\tau-1)}\right\|_2^2 - 2 \left\langle\rvu_{k}^{(\tau-1)}, \rvx_{n}\right\rangle,
\end{align}
 the reestimation of the cluster centers is performed as
\begin{align}
\label{eq:center_update}
\forall k\in\intint{K}, \rvu^{(\tau)}_k \leftarrow \hat{\rvx}_k(\rvt^{(\tau)}) \eqdef \frac{1}{n_k^{(\tau)}} \sum_{n: t^{(\tau)}_n= k} {\rvx_{n}}
\end{align}
where $n_k^{(\tau)}\eqdef |\{n: t^{(\tau)}_n=k\}|$ is the number of points in cluster $k$
at time $\tau$ and $\hat{\rvx}_k(\rvt)$ is the mean vector of the elements of cluster $k$ according to assignment $\rvt$. 

\paragraph{Complexity of Lloyd's algorithm.} The assignment step \eqref{eq:assignment_problem_kmeans} costs $\mathcal{O}(\nexamples\datadim\nclusters)$ operations while the update of the centers~\eqref{eq:center_update} costs $\mathcal{O}\left (\nexamples\datadim\right )$ operations. Hence, the bottleneck of the overall time complexity $\mathcal{O}(\nexamples\datadim\nclusters)$ stems from the assignment step. Once the clusters have been defined, assigning $\nexamples'$ new points to these clusters is performed via \eqref{eq:assignment_problem_kmeans} at the cost of $\mathcal{O}\left (\nexamples'\datadim\nclusters \right )$ operations.

The main contribution in this paper relies on the idea that \eqref{eq:assignment_problem_kmeans} may be computed more efficiently by approximating $\rmU$ as a fast operator.

\subsection{Learning Fast Transforms as the Product of Sparse Matrices}
\label{sec:palm4msa}

\paragraph{Structured linear operators as products of sparse matrices.}
The popularity of some linear operators from $\R^{M}$ to $\R^{M}$ (with $M<\infty$)
 like Fourier or Hadamard transforms comes from both their mathematical 
 properties and their ability to compute the mapping of some input $\rvx\in\R^M$ with efficiency, typically in $\mathcal{O}\left (M\log\left (M\right )\right )$ rather than 
 in $\mathcal{O}\left (M^2\right)$ operations .
The main idea of the related fast algorithms is that the matrix $\rmU\in\sR^{M\times M}$ characterizing such linear operators can be written as the product $\rmU=\Pi_{q\in\intint{\nfactors}}\rmS_q$ of $\nfactors$ sparse matrices $\rmS_q$, with $Q=\mathcal{O}\left (\log M\right )$ factors and $\left \|\rmS_q\right \|_0=\mathcal{O}\left (M\right )$ non-zero coefficients per factor \cite{LeMagoarou2016Flexible,Morgenstern1975Linear}:
for any vector $\rvx\in\sR^M$, $\rmU\rvx$ can thus be computed as $\mathcal{O}\left (\log M\right )$ products $\rmS_0 \left (\rmS_1 \left (\ldots \left (\rmS_{Q-1}\rvx\right )\right )\right )$ between a sparse matrix and a vector, the cost of each product being $\mathcal{O}\left (M\right )$. This gives a $\mathcal{O}(M \log M)$ time complexity for computing $\rmU\rvx$ in that case.

\paragraph{Learning a computationally-efficient decomposition approximating an arbitrary operator.} When the linear operator $\rmU$ is an arbitrary matrix, one may approximate it with such a sparse-product structure by learning the factors $\left \lbrace\rmS_q\right \rbrace_{q\in\intint{Q}}$ in order to benefit from a fast algorithm.
A recent contribution~\cite{LeMagoarou2016Flexible} has proposed algorithmic strategies to learn such a factorization. Based on the proximal alternating linearized minimization (\texttt{PALM}) algorithm~\cite{bolte2014proximal}, the \texttt{PALM} for Multi-layer Sparse Approximation (\palm) algorithm~\cite{LeMagoarou2016Flexible} aims at approximating a matrix $\rmU\in\sR^{\nclusters\times\datadim}$ as a product of sparse matrices by solving
\begin{align}
\label{eq:palm4msa}
\min_{\left \lbrace\rmS_q\right \rbrace_{q\in\intint{Q}}} \left \|\rmU -  \prod_{q\in\intint{\nfactors}}{\rmS_q}\right \|_F^2 + \sum_{q\in\intint{\nfactors}} \delta_{\mathcal{E}_q}(\rmS_q)
\end{align}
where, for each $q\in\intint{Q}$, $\delta_{\mathcal{E}_q}(\rmS_q)=0$ 
if $\rmS_q \in \mathcal{E}_q$ and $\delta_{\mathcal{E}_q}(\rmS_q)=+\infty$ otherwise, $\mathcal{E}_q$ being a constraint set that typically impose a sparsity structure on its elements, as well as a scaling constraint. The \palm algorithm and more related details are given in Appendix~\ref{sec:app:palm4msa}.

Although this problem is non-convex and the computation of a global optimum cannot be
ascertained, the \palm algorithm is able to find good local minima with convergence guarantees.

%!TEX root=neurips2019_qmeans.tex

\section{QuicK-means}
\label{sec:contribution}

We here introduce our main contribution, \texttt{QuicK-means} (abbreviated by \qkmeans), 
show its convergence property and analyze its computational complexity.

\subsection{\qkmeans: Encoding Centroids as Products of Sparse Matrices}

\texttt{QuicK-means} is a variant of the \kmeans algorithm in which the matrix of centroids $\rmU$
is approximated as a product $\rmV=\prod_{\in\intint{\nfactors}}\rmS_q$ of sparse matrices $\rmS_q$.
Doing so will allow us to cope with the computational bulk imposed by the product $\rmU\rvx$
(cf.~\eqref{eq:assignment_problem_kmeans}) at the core of the cluster assignment process.

Building upon the \kmeans optimization problem~\eqref{eq:kmean_problem} and fast-operator approximation problem~\eqref{eq:palm4msa} the \qkmeans optimization problem 
writes:
\begin{align}
\label{eq:qmean_problem}
 \argmin_{\rmS_1, \ldots, \rmS_{\nfactors}, \rvt} g\left(\rmS_1, \ldots, \rmS_{\nfactors}, \rvt\right)
    \eqdef \sum_{k\in\intint{\nclusters}}\sum_{n: t_n = k} \left\|\rvx_n -\rvv_k\right\|^2 + \sum_{q\in\intint{\nfactors}} \delta_{\mathcal{E}_q}(\rmS_q) \text{ s. t. } \rmV = \prod_{q\in\intint{\nfactors}}\rmS_q
\end{align}
This is a regularized version of the \kmeans optimization problem~\eqref{eq:kmean_problem} in which centroids $\rvv_k$ are constrained to form a matrix $\rmV$ with a fast-operator structure, the indicator functions $\delta_{\mathcal{E}_q}$ imposing the sparsity of matrices $\rmS_q$.
More details on the modeling choices are given in the experimental part in section~\ref{sec:uses:settings}.

This problem can be solved using Algorithm~\ref{algo:qmeans},
which proceeds in a similar way as Lloyd's algorithm by alternating an assignment step at line \ref{line:qmeans:assignment} and an update of the centroids at lines~\ref{line:qmeans:compute_means}--\ref{line:qmeans:U}. The assignment step can be computed efficiently thanks to the fast-structure in matrix $\rmV$. The update of the centroids relies on learning a fast-structure operator $\rmV$ that approximate of the true centroid matrix $\rmU$ weighted by the number of examples $n_k$ assigned to each cluster $k$.

\begin{algorithm}[t]
	\caption{\qkmeans algorithm and its time complexity.}
	\label{algo:qmeans}
	\begin{algorithmic}[1]
\REQUIRE $\rmX \in \R^{\nexamples \times \datadim}$, $\nclusters$, initialization $\left \lbrace \rmS_q^{(0)} : \rmS_q^{(0)} \in \mathcal{E}_q\right \rbrace_{q\in\intint{\nfactors}}$
\COMMENT{$A \eqdef \min\left (\nclusters, \datadim\right )$}
%\STATE $\rmV^{(0)} \eqdef \prod_{q\in\intint{\nfactors}}{\rmS_q^{(0)}}$
\STATE Set $\rmV^{(0)} : \rvx \mapsto \prod_{q\in\intint{\nfactors}}{\rmS_q^{(0)}} \rvx$
\COMMENT{$B \eqdef \max\left (\nclusters, \datadim\right )$}
\FOR{$\tau=1, 2, \ldots$ until convergence}
	\STATE $\rvt^{(\tau)} \eqdef \argmin_{\rvt \in \intint{\nclusters}^\nexamples} \sum_{n\in\intint{\nexamples}} {\left \|\rvx_n - \rvv^{(\tau -1)}_{t_n}\right \|^2}$
	\COMMENT{$\mathcal{O}\left (\nexamples\left(A\log A+B\right ) + AB\right )$}
	\label{line:qmeans:assignment}
	\STATE $\forall k\in\intint{\nclusters}, \rvu_k \eqdef \frac{1}{n_k} \sum_{n: t_n^{(\tau)}= k} {\rvx_n}$
with $n_k \eqdef |\{n: t_n^{(\tau)}=k\}|$
	\COMMENT{$\bigO{\nexamples\datadim}$}
	\label{line:qmeans:compute_means}
	\STATE $\rmA \eqdef \rmD_{\sqrt{\rvn}} \times \rmU $
	\COMMENT{$\bigO{\nclusters\datadim}$}
	\label{line:qmeans:A}
	\STATE $\mathcal{E}_0 \eqdef \left \lbrace \rmD_{\sqrt{\rvn}} \right \rbrace$
	\label{line:qmeans:E0}
	\STATE $\left \lbrace \rmS_q^{(\tau)}\right \rbrace_{q=0}^\nfactors \eqdef \argmin_{\left \lbrace \rmS_q\right \rbrace_{q=0}^\nfactors} \left \|\rmA - \prod_{q=0}^\nfactors\rmS_q\right \|_F^2 + \sum_{q=0}^\nfactors \delta_{\mathcal{E}_q}(\rmS_q)$\\
	\COMMENT{$\bigO{AB\left (\log^2 A+\log B\right )}$ (or $\bigO{AB\left (\log^3A+\log A \log B\right )}$)}
	\label{line:qmeans:S}
	\STATE Set $\rmV^{(\tau)} : \rvx \mapsto \prod_{q\in\intint{\nfactors}}{\rmS_q^{(\tau)}} \rvx$
	\COMMENT{$\bigO{1}$}
	\label{line:qmeans:U}
	\ENDFOR
	\ENSURE assignement vector $\rvt$ and sparse matrices $\left \lbrace \rmS_q : \rmS_q \in \mathcal{E}_q\right \rbrace_{q\in\intint{\nfactors}}$ such that $\prod_{q\in\intint{\nfactors}}\rmS_q \approx \rmU$ the $\nclusters$ means of the $\nexamples$ data points
\end{algorithmic}
\end{algorithm}

\iffalse
\begin{remark}[Assignment/Re-estimation trade-off.]
A strategy to tackle this problem would be to first run the vanilla K-means algorithm,
 obtain the matrix of centroids $U$ and then encode $U$ as a product of sparse matrices
 using Hierarchical Palm4MSA. This would however prevent us from taking advantage of 
 the expected low complexity product that plays a role in the assignement step of 
 the procedure.
\end{remark}

\todo[inline]{At some point, talk about the trade-off that we are playing with
regarding the cost of the assignment and the cost of the re-estimation procedure.}
\fi

\subsection{Convergence of \qkmeans}
Similarly to \kmeans, \qkmeans converges locally as stated in the following proposition.

\begin{proposition}[Convergence of \qkmeans]
\label{thm:convergence}
The iterates $\left \lbrace\rmS^{(\tau)} \right \rbrace_{q\in\intint{\nfactors}}$ and $\rvt^{(\tau)}$ in Algorithm~\ref{algo:qmeans} are such that the values
\begin{align}
\label{eq:qmean_problem_2}
    g(\rmS_1^{(\tau)}, \ldots,\rmS_\nfactors^{(\tau)}, \rvt^{(\tau)})
    = \sum_{k\in\intint{\nclusters}} \sum_{n: \rvt^{(\tau)}_n = k} \norm{\rvx_n - \rvv^{(\tau)}_k}^2 + \sum_{q\in\intint{\nfactors}} \delta_{\mathcal{E}_q}\left (\rmS_q^{(\tau)}\right )
    \text{ s.t. } \rmV = \prod_{q\in\intint{\nfactors}}{\rmS_q^{(\tau)}}
\end{align}
of the objective function are non-increasing.
\end{proposition}

\begin{proof}
To proove this convergence, we show that each of the assignment and centroid update steps in one iteration $\tau$ of the algorithm actually reduces the overall objective.

\paragraph{Assignment step (Line \ref{line:qmeans:assignment})} For a fixed $\rmV^{(\tau-1)}$, the optimization problem at Line \ref{line:qmeans:assignment} is separable for each example indexed by $n \in \intint{\nexamples}$ and the new indicator vector $\rvt^{(\tau)}$ is thus defined as:
\begin{align}
\label{eq:qmean_problem_U_fixed}
 t^{(\tau)}_n = \argmin_{k \in \intint{\nclusters}} \norm{\rvx_n - \rvv_k^{(\tau-1)}}_2^2.
\end{align}
This step minimizes the first term in~\eqref{eq:qmean_problem_2} w.r.t. $\rvt$ while the second term is constant so we have 
\begin{align*}
g(\rmS_1^{(\tau-1)}, \ldots,\rmS_\nfactors^{(\tau-1)}, \rvt^{(\tau)}) \leq g(\rmS_1^{(\tau-1)}, \ldots,\rmS_\nfactors^{(\tau-1)}, \rvt^{(\tau-1)}).
\end{align*}

\paragraph{Centroids update step (Lines \ref{line:qmeans:compute_means}--\ref{line:qmeans:U}).} We know consider a fixed assignment vector $\rvt$. We first note that for any cluster $k$ with true centroid $\rvu_k$ and approximated centroid $\rvv_k$, we have
\begin{align*}
	\sum_{n: t_n = k} \norm{\rvx_n -\rvv_k}^2
	 & =\sum_{n: t_n = k} \norm{\rvx_n -\rvu_k+\rvu_k-\rvv_k}^2\\
	& = \sum_{n: t_n = k}\left(\norm{\rvx_n-\rvu_k}^2+\norm{\rvu_k-\rvv_k}^2-2\langle\rvx_n-\rvu_k,\rvu_k-\rvv_k\rangle\right)\notag\\
	& = \sum_{n: t_n= k} \norm{\rvx_n-\rvu_k}^2+n_k\norm{\rvu_k-\rvv_k}^2 - 2 \left\langle\underbrace{\sum_{n: t_n = k}\left (\rvx_n-\rvu_k\right )}_{=0},\rvu_k-\rvv_k\right\rangle\notag\\
	&= \sum_{n: t_n = k} \norm{\rvx_n-\rvu_k}^2 + \norm{\sqrt{n_k}\left (\rvu_k-\rvv_k\right )}^2
\end{align*}

For a fixed $\rvt$, the new sparsely-factorized centroids are solutions of the following subproblem:
\begin{align}
 \argmin_{\rmS_1, \ldots,\rmS_Q} g(\rmS_1, \ldots,\rmS_Q, \rvt) 
 & = \argmin_{\rmS_1, \ldots,\rmS_Q} \sum_{k\in\intint{\nclusters}}  \sum_{n: t_n = k} \norm{\rvx_n - \rvv_k}^2_2 + \sum_{q\in\intint{\nfactors}} \delta_q(\rmS_q) 
 \text{ s. t. } \rmV = \prod_{q\in\intint{\nfactors}}{\rmS_q} \nonumber \\
 & = \argmin_{\rmS_1, \ldots,\rmS_Q} \norm{\rmD_{\sqrt{\rvn}} (\rmU - \rmV)}_F^2
 + \sum_{k\in\intint{\nclusters}} c_k + \sum_{q\in\intint{\nfactors}} \delta_q(\rmS_q)
 \text{ s. t. } \rmV = \prod_{q\in\intint{\nfactors}}{\rmS_q} \nonumber\\
  & = \argmin_{\rmS_1, \ldots,\rmS_Q} \norm{\rmA - \rmD_{\sqrt{\rvn}} \prod_{q\in\intint{\nfactors}}{\rmS_q}}_F^2
 + \sum_{q\in\intint{\nfactors}} \delta_q(\rmS_q)
 \label{eq:qmeans_problem_t_fixed}
\end{align}
where :
\begin{itemize}
 \item $\sqrt{\rvn} \in \R^{\nclusters}$ is the pair-wise square root of the vector indicating the number of observations $n_k \eqdef \left | \left \lbrace n: t_n = k\right \rbrace \right |$  in each cluster $k$;
 \item $\rmD_{\sqrt{\rvn}} \in \R^{K \times K}$ refers to a diagonal matrix with vector $\sqrt{\rvn}$ on the diagonal;
 \item $\rmU\in \R^{K \times d}$ refers to the unconstrained centroid matrix obtained from the data matrix $\rmX$ and the indicator vector $\rvt$: $\rvu_k \eqdef \frac{1}{n_k}\sum_{n:t_n = k} {\rvx_n}$ (see Line~\ref{line:qmeans:compute_means});
 \item $\rmD_{\sqrt{\rvn}} (\rmU - \rmV)$ is the matrix with $\sqrt{n_k}\left (\rvu_k-\rvv_k\right )$ as $k$-th row;
 \item $c_k \eqdef \sum_{n: t_n = k}\norm{\rvx_n - \rvu_k}$ is constant w.r.t. $ \rmS_1, \ldots,\rmS_Q$;
 \item $\rmA \eqdef \rmD_{\sqrt{\rvn}} \rmU$ is the unconstrained centroid matrix reweighted by the size of each cluster (see Line~\ref{line:qmeans:A}).
\end{itemize}

A local minimum of~\eqref{eq:qmeans_problem_t_fixed} is obtained by applying the \palm algorithm or its hierarchical variant to approximate $\rmA$, as in Line~\ref{line:qmeans:S}. The first factor is forced to equal $\rmD_{\sqrt{\rvn}}$ by setting $\mathcal{E}_0$ to a singleton at Line~\ref{line:qmeans:E0}. Using the previous estimate $\left \lbrace \rmS_q^{(\tau-1)}\right \rbrace_{q\in\intint{\nfactors}}$ to initialize this local minimization, we thus obtain that $g(\rmS_1^{(\tau)}, \ldots,\rmS_\nfactors^{(\tau)}, \rvt^{(\tau)}) \leq g(\rmS_1^{(\tau-1)}, \ldots,\rmS_\nfactors^{(\tau-1)}, \rvt^{(\tau)})$.

We finally have, for any $\tau$,
\begin{align*}
%g\left (\left \lbrace \rmS_q^{(\tau)}\right \rbrace_{q\in\intint{\nfactors}}, \rvt^{(\tau)}\right ) & \leq
%g\left (\left \lbrace \rmS_q^{(\tau-1)}\right \rbrace_{q\in\intint{\nfactors}}, \rvt^{(\tau)}\right ) \leq
%g\left (\left \lbrace \rmS_q^{(\tau-1)}\right \rbrace_{q\in\intint{\nfactors}}, \rvt^{(\tau-1)}\right ) \\
%& \leq
%\ldots \leq
%g\left (\left \lbrace \rmS_q^{(0)}\right \rbrace_{q\in\intint{\nfactors}}, \rvt^{(0)}\right )
%\\
g\left (\rmS_1^{(\tau)}, \ldots,\rmS_\nfactors^{(\tau)}, \rvt^{(\tau)}\right ) 
& \leq
g\left (\rmS_1^{(\tau-1)}, \ldots,\rmS_\nfactors^{(\tau-1)}, \rvt^{(\tau)}\right )
\leq 
g\left (\rmS_1^{(\tau-1)}, \ldots,\rmS_\nfactors^{(\tau-1)}, \rvt^{(\tau-1)}\right ) \\
& \leq \ldots \leq
g\left (\rmS_1^{(0)}, \ldots,\rmS_\nfactors^{(0)}, \rvt^{(0)}\right )
\end{align*}
\end{proof}

\subsection{Complexity analysis}

Since the space complexity of the proposed \qkmeans algorithm is comparable to that of \kmeans, we only detail its time complexity. We set $A=\min\left (\nclusters, \datadim\right )$ and $B=\max\left (\nclusters, \datadim\right )$, and assume that the number of factors satisfies $\nfactors=\bigO{\log A}$.

The analysis is proposed under the following assumptions: the product between two dense matrices of shapes ${N_1\times N_2}$ and ${N_2\times N_3}$ can be done $\mathcal{O}\left (N_1 N_2 N_3 \right )$ operations; 
the product between a sparse matrix with $\bigO{S}$ non-zero entries and a dense vector can be done in $\bigO{S}$ operations; 
the product between two sparse matrices of shapes ${N_1\times N_2}$ and ${N_2\times N_3}$, both having $\bigO{S}$ non-zero values can be done in $\bigO{S \min\left (N_1, N_3\right )}$ and the number of non-zero entries in the resulting matrix is $\bigO{S^2}$.

\paragraph{Complexity of the \kmeans algorithm.}
We recall here that the \kmeans algorithm complexity is dominated by its cluster assignation step which requires $\bigO{\nexamples\nclusters\datadim}=\bigO{\nexamples A B}$ operations (see Eq.~\eqref{eq:assignment_problem_kmeans}).

\paragraph{Complexity of algorithm \palm.} The procedure consists in an alternate optimization of each sparse factor. 
At each iteration, the whole set of $\nfactors$ factors is updated with at a cost in $\bigO{AB\left (\log^2 A+\log B\right )}$, as detailed in Appendix~\ref{sec:app:palm4msa}. 
The bottleneck is the computation of the gradient, which benefits from fast computations with sparse matrices.
The hierarchical version of \palm proposed in~\cite{LeMagoarou2016Flexible} consists in running $\palm$ $2Q$ times so that its time complexity is in $\bigO{AB\left (\log^3 A + \log A \log B\right )}$.

\paragraph{Complexity of the \qkmeans algorithm.} The overall complexity of \qkmeans is in $\bigO{\nexamples\left(A\log A+B\right ) + AB \log^2 A}$ when used with \palm and in $\bigO{\nexamples\left(A\log A+B\right ) + AB \log^3 A}$ when used with the hierarchical version of \palm. The time complexities of the main steps are given in Algorithm~\ref{algo:qmeans}. 

The assignation step (line~\ref{line:qmeans:assignment} and Eq.~\eqref{eq:assignment_problem_kmeans}) benefits from the fast computation of $\rmV \rmX$ in~$\bigO{\nexamples\left(A\log A+B\right )}$ while the computation of the norms of the cluster centers is in $\bigO{AB}$.
One can see that the computational bottleneck of \kmeans is here reduced, which shows the advantage of using \qkmeans when $\nexamples$, $\nclusters$ and $\datadim$ are large.

The computation of the centers of each cluster, given in line~\ref{line:qmeans:compute_means}, is the same as in \kmeans and takes $\bigO{\nexamples\datadim}$ operations.

The update of the fast transform, in lines~\ref{line:qmeans:A} to~\ref{line:qmeans:U} is a computational overload compared to \kmeans. 
Its time complexity is dominated by the update of the sparse factors at line~\ref{line:qmeans:S}, in $\bigO{AB \log^2 A}$ if \palm is called and in $\bigO{AB \log^3 A}$ if its hierarchical version is called. 
Note that this cost is dominated by the cost of the assignement step as soon as the number of examples $\nexamples$ is greater than $\log^3 A$.

\section{Experiments and applications}
\label{sec:uses}

\subsection{Experimental setting}
\label{sec:uses:settings}

\paragraph{Implementation details.}
The simulations have been conducted in Python, including for the \palm algorithm.
Running times are measured on computer grid with 3.8GHz-CPUs (2.5GHz in Figure~\ref{fig:time_csr}).
Fast operators $\rmV$ based on sparse matrices $\rmS_q$ are implemented with \texttt{csr\_matrix} objects from the \texttt{scipy.linalg} package. 
While more efficient implementations may be beneficial for larger deployment, our implementation is sufficient as a proof of concept for assessing the performance of the proposed approach. 
In particular, the running times of fast operators of the form $\prod_{q\in\intint{\nfactors}}{\rmS_q}$ have been measured when applying to random vectors, for several sparsity levels: 
as shown in Figure~\ref{fig:time_csr}, they are significantly faster than dense operators -- implemented as a \texttt{numpy.ndarray} matrix --, especially when the data size is larger than $10^3$.

\begin{figure}[tbh]
\centering
\includegraphics[width=.8\textwidth]{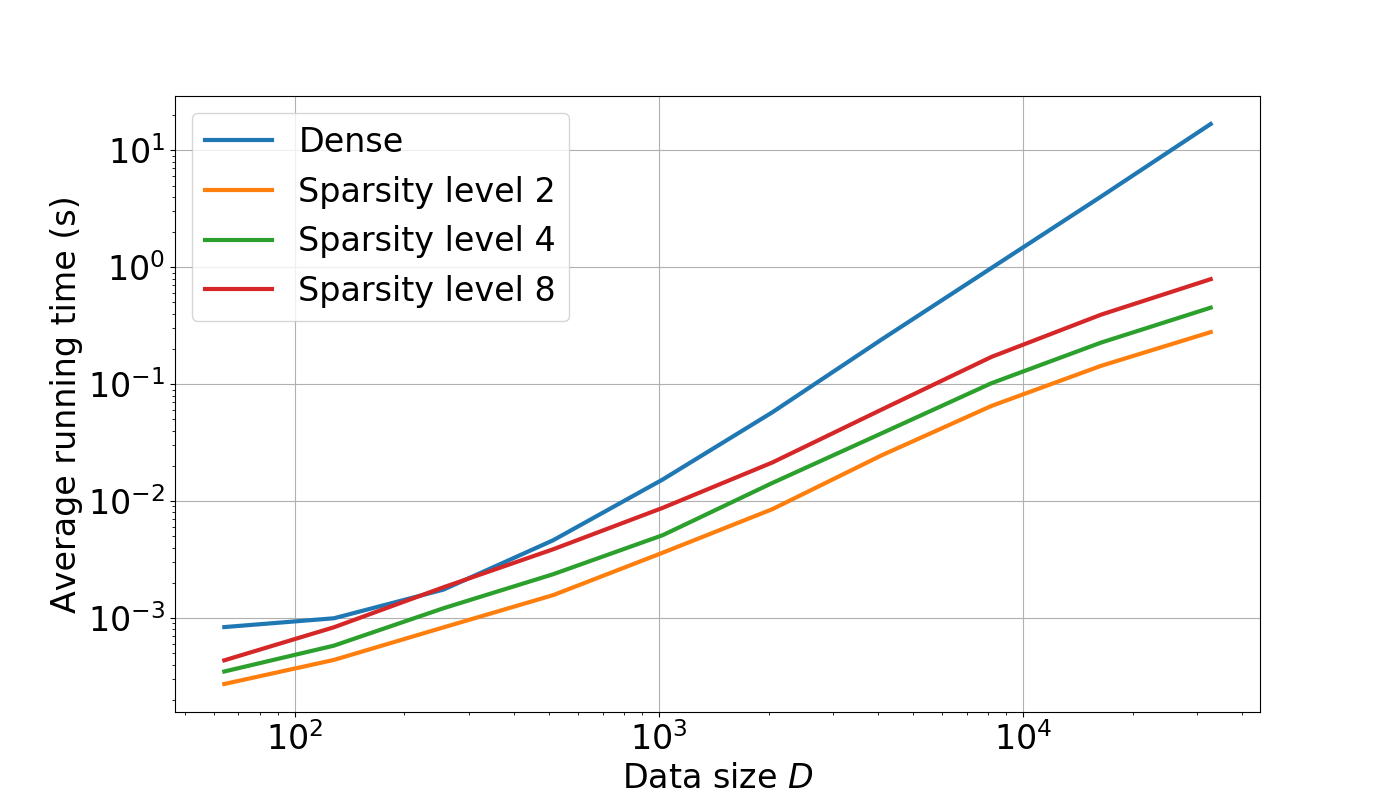}
\caption{Running times, averaged over 30 runs, when applying dense or fast $\datadim \times \datadim$ operators to a set of 100 random vectors. The number of factors in fast operators equals $\log_2\left (\datadim\right )$ and the sparsity level denotes the number of non-zero coefficients per row and per column in each factor.}
\label{fig:time_csr}
\end{figure}

\paragraph{Datasets.}
We present results on real-world and toy datasets summarized in Table \ref{table:data}. On the one hand, the real world datasets \texttt{MNIST}~\cite{lecun-mnisthandwrittendigit-2010} and \texttt{Fashion-Mnist}~\cite{Pedregosa2011Scikit} 
are used to show --- quantitatively and qualitatively --- the good quality of the obtained centroids when using our method \qkmeans. On the other hand, we use the \texttt{blobs} synthetic dataset from \texttt{sklearn.dataset} to show the speed up offered by our method \qkmeans when the number of clusters and the dimensionality of the data are sufficiently large.

\begin{table*}[!h]
\centering
\begin{tabular}{|c|c|c|c|c|c|}
\hline
\textbf{Dataset} & \textbf{Data dim.} $\datadim$        & \textbf{\# classes} & \textbf{Training set size} $\nexamples$ & \textbf{Test set size} $\nexamples'$ \\ \hline
MNIST                   & 784   & 10        & 60 000    & 10 000               \\ \hline
Fashion-MNIST           & 784   & 10        & 60 000    & 10 000               \\ \hline
Blobs (clusters std: 12)   & 2000  & 1000      & 29000      & 1000               \\ \hline
\end{tabular}
\caption{Datasets statistics}
\label{table:data}
\end{table*}

\paragraph{Algorithm settings.} 
The \qkmeans algorithm is used with $Q\eqdef\log_2\left (A\right )$ sparse factors, where  $A=\min\left (\nclusters, \datadim\right )$. 
All factors $\rmS_q$ are with shape $A \times A$ except, depending on the shape of $\rmA$, the leftmost one ($\nclusters\times A$) or the rightmost one ($A\times\datadim$). 
The sparsity constraint of each factor $\rmS_q$ is set in $\mathcal{E}_q$ and is governed by a global parameter denoted as \textit{sparsity level}, which indicates the desired number of non-zero coefficients in each row and in each column of $\rmS_q$. 
Since the projection onto this set of structured-sparsity constraints may be computationally expensive, this projection is relaxed in the implementation of \palm and only guarantees that the number of non-zero coefficients in each row and each column is at least the sparsity level, as in~\cite{LeMagoarou2016Flexible}.
The actual number of non-zero coefficients in the sparse factors is measured at the end of the optimization process and reported in the results.
The sparse factors are updated using the \palm rather than its hierarchical version, since we observed that this was a better choice in terms of computational cost, with satisfying approximation results (See Figure~\ref{fig:mnist:objfun}~and~\ref{fig:fmnist:objfun}).
Additional details about \palm are given in Appendix~\ref{sec:app:palm4msa}.
The stopping criterion of \kmeans and \qkmeans consists of a tolerance set to $10^{-6}$ on the relative variation of the objective function and a maximum number of iterations set to 10 for the \texttt{Blobs}dataset and to 20 for others. The same principle governs the stopping criterion of \palm with a tolerance set to $10^{-6}$ and a maximum number of iterations set to 300. Each experiment have been replicated using different seed values for random initialisation. Competing techniques share the same seed values, hence share the same initialisation of centroids.

\subsection{Clustering}

\begin{figure}
\begin{subfigure}[b]{.49\textwidth}
\includegraphics[width=\textwidth]{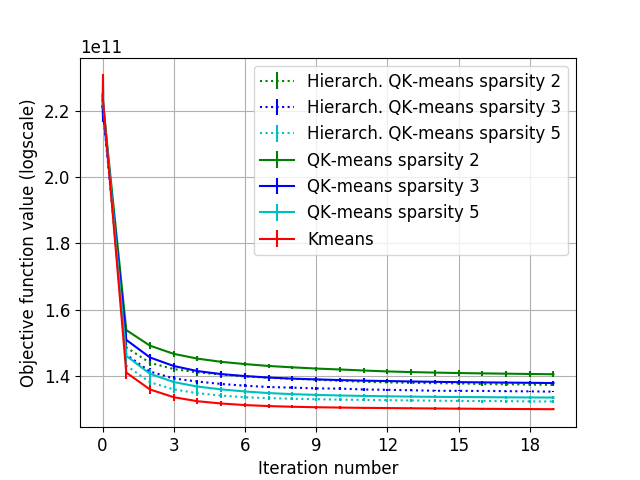}
\caption{MNIST, $\nclusters=30$: objective function.}
\label{fig:mnist:objfun}
\end{subfigure}
\begin{subfigure}[b]{.49\textwidth}
\includegraphics[width=\textwidth]{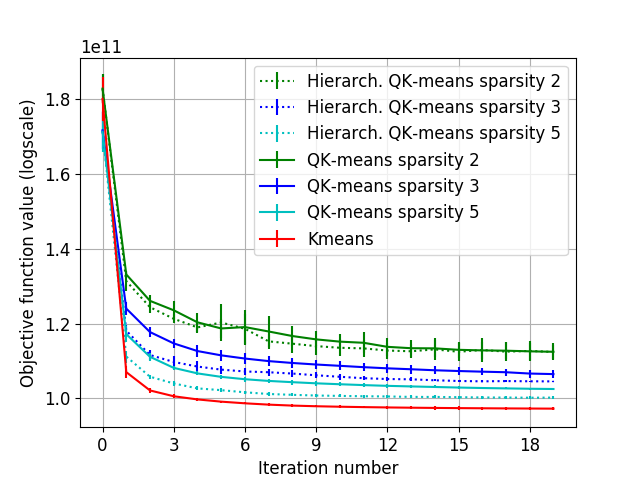}
\caption{Fashion-MNIST, $\nclusters=30$: objective function.}
\label{fig:fmnist:objfun}
\end{subfigure}
\begin{subfigure}[t]{.49\textwidth}
\includegraphics[width=\textwidth]{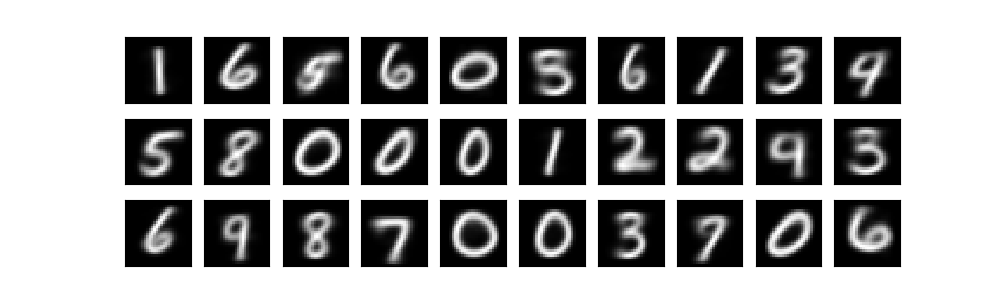}
\caption{\kmeans centroids.}
\label{fig:mnist:kmeans:centroids}
\end{subfigure}
\begin{subfigure}[t]{.49\textwidth}
\includegraphics[width=\textwidth]{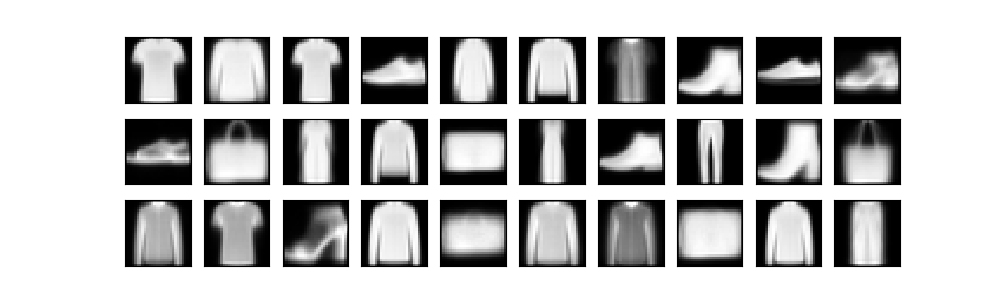}
\caption{\kmeans centroids.}
\label{fig:fmnist:kmeans:centroids}
\end{subfigure}
\begin{subfigure}[t]{.49\textwidth}
\includegraphics[width=\textwidth]{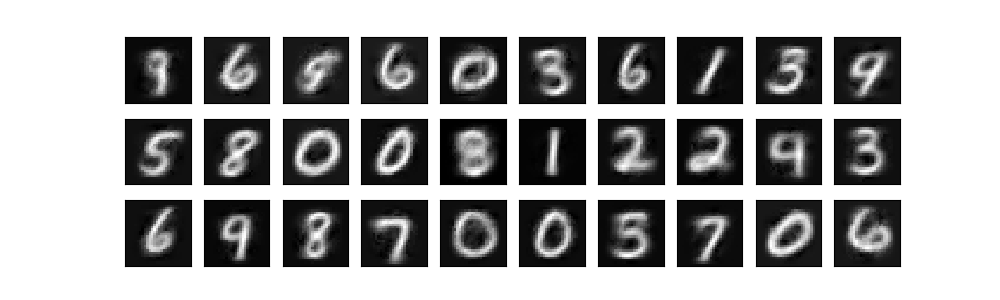}
\caption{\qkmeans centroids.}
\label{fig:mnist:qkmeans:centroids}
\end{subfigure}
\begin{subfigure}[t]{.49\textwidth}
\includegraphics[width=\textwidth]{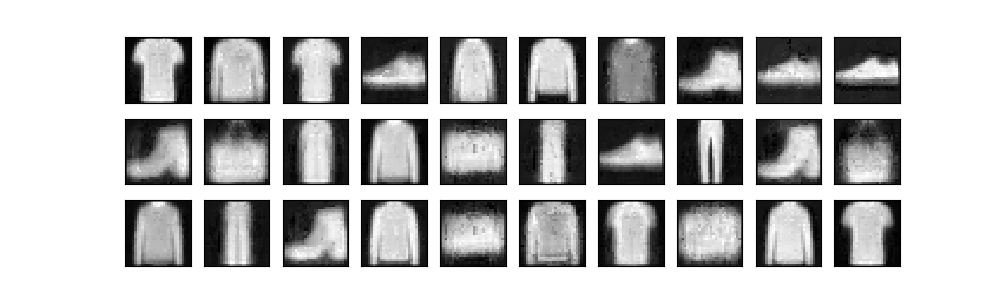}
\caption{\qkmeans centroids.}
\label{fig:fmnist:qkmeans:centroids}
\end{subfigure}
\begin{subfigure}[t]{.49\textwidth}
\includegraphics[width=\textwidth]{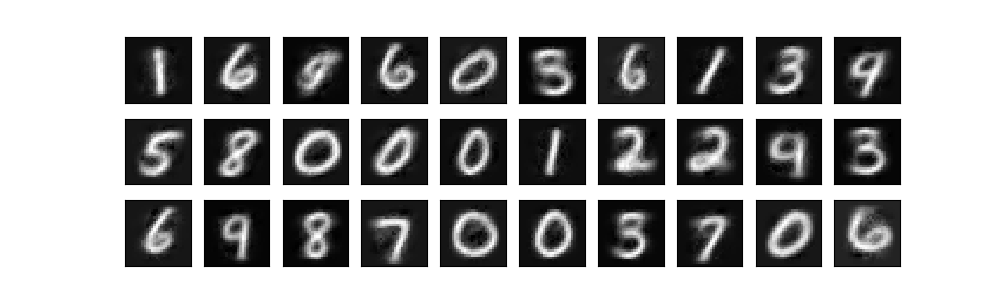}
\caption{Hierarchical-\palm \qkmeans centroids.}
\label{fig:mnist:hqkmeans:centroids}
\end{subfigure}
\begin{subfigure}[t]{.49\textwidth}
\includegraphics[width=\textwidth]{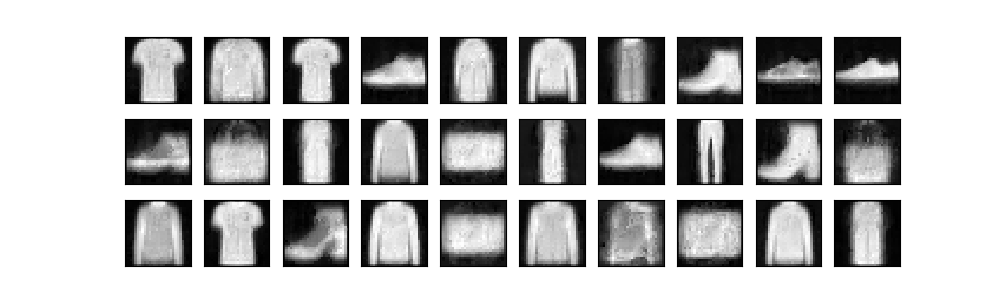}
\caption{Hierarchical-\palm \qkmeans centroids.}
\label{fig:fmnist:hqkmeans:centroids}
\end{subfigure}
\caption{Clustering results on MNIST (left) and Fashion-MNIST (right) for $\nclusters=30$ clusters.}
\label{fig:clustering:realdata}
\end{figure}

\paragraph{Approximation quality.} One important question is the ability of the fast-structure model to fit arbitrary data.
Indeed, no theoretical result about the expressivity of such models is currently available.
In order to assess this approximation quality, the MNIST and Fashion-MNIST data have been clustered into $\nclusters=30$ clusters by \kmeans, \qkmeans and a variant of \qkmeans using the hierarchical version of \palm, with several sparsity levels.
Results are reported in Figure~\ref{fig:clustering:realdata}.
In Figures~\ref{fig:mnist:objfun} and~\ref{fig:fmnist:objfun}, one can observe that the objective function of \qkmeans is decreasing in a similar way as \kmeans over iterations.
In particular, the use of the fast-structure model does not seem to increase the number of iteration necessary before convergence.
At the end of the iterations, the value of objective function for \qkmeans is slightly above that of \kmeans.
As expected, the sparser the model, the more degradation in the objective function.
However, even very sparse models do not degrade the results significantly. These Figures also demonstrate the convergence property of the \qkmeans algorithm when using the standard, proved convergent, \textit{Palm4MSA} algorithm: in this case, the objective function is always non-increasing whereas the \qkmeans version with \textit{Hiearchical Palm4MSA}, not guaranteed to converge, suffers a small bump in its objective function (see Figure~\ref{fig:fmnist:objfun} iteration~6).
The approximation quality can be assessed visually, in a more subjective and interpretable way, in Figures~\ref{fig:mnist:kmeans:centroids} to~\ref{fig:fmnist:hqkmeans:centroids} where the obtained centroids are displayed as images.
Although some degradation may be observed in some images, one can note that each image obtained with \qkmeans clearly represents a single visual item without noticeable interference with other items.

\paragraph{Clustering assignation times.}
Higher dimensions are required to assess the computational benefits of the proposed approach, as shown here.
The assignation times of the clustering procedure were measured on the \texttt{Blobs} dataset.
The centroid matrices are with shape $\nclusters \times \datadim$ with $\datadim=2000$  and $\nclusters\in\left \lbrace 128, 256, 512\right \rbrace$.
Results reported in Figure~\ref{fig:clustering:blobs:assignation_time} show that in this setting and with the current implementation, the computational advantage of \qkmeans is observed in high dimension, for $\nclusters=256$ and $\nclusters=512$ clusters. It is worth noticing that when $K$ increases, the running times are not affected that much for \qkmeans while it significantly grows for \kmeans. These trends are directly related to the number of model parameters that are reported in the figure.

\begin{figure}[tbh]
\centering
\includegraphics[width=.8\textwidth]{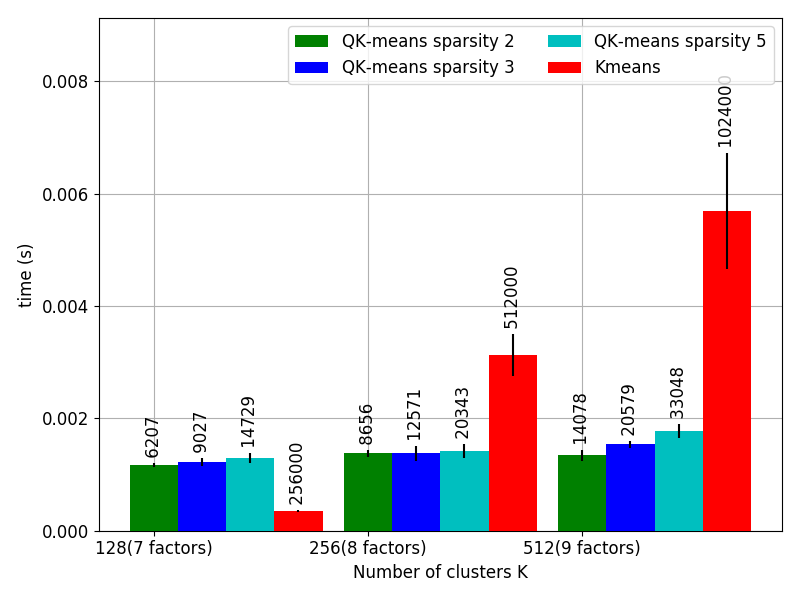}
\caption{Clustering Blobs data: running times of the assignation step, averaged over 5 runs. The vertical black lines are the standard deviation w.r.t. the runs and the average number of parameters actually learned  in the models are reported above those lines.\addVE{to be completed}.}
\label{fig:clustering:blobs:assignation_time}
\end{figure}

\subsection{Nearest-neighbor search in a large dataset}
The Nearest-neighbor search is a fundamental task that suffers from computational limitations when the dataset is large.
Fast strategies have been proposed, e.g., using kd trees or ball trees.
One may also use a clustering strategy to perform an approximate nearest-neighbor search: the query is first compared to $\nclusters$ centroids computed beforehand by clustering the whole dataset, and the nearest neighbor search is then performed among a lower number of data points, within the related cluster.
We compare this strategy using \kmeans and \qkmeans against the \texttt{scikit-learn} implementation~\cite{Pedregosa2011Scikit} of the nearest-neighbor search (brute force search, kd tree, ball tree).
Inference time results on the \texttt{Blobs} dataset are reported in Figure~\ref{fig:nn:blobs} and accuracy results are displayed in Table~\ref{table:results_blobs}. 
The running times reported in Figure~\ref{fig:nn:blobs} show a dramatic advantage of using a clustering-based approximate search 
and this advantage is even stronger with the clustering obtained by our \qkmeans method. This speed-up comes at a cost though, we can see a drop in classification performance in Table~\ref{table:results_blobs}. 

\begin{figure}[tbh]
\centering
\includegraphics[width=\textwidth]{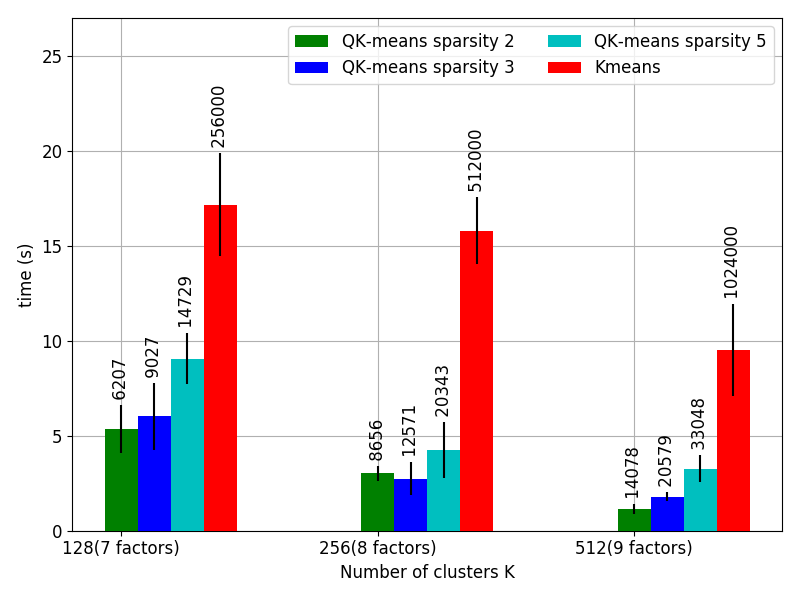}
\label{fig:nn:blobs:times}

\caption{Running time of nearest neighbor search on blobs data. Results are averaged over 5 runs (vertical lines: standard deviation) and the average number of parameters actually learned is reported above each bar. The results for the Brute Force Search, KD Tree and Ball Tree are not displayed because they were longer than 10 times the K-means search version.}
\label{fig:nn:blobs}
\end{figure}

\begin{table}[]

\centering
\begin{tabular}{@{}l|P{2.5cm}}
\toprule
                                    & Accuracy \texttt{Blobs} \\ \midrule
1NN Brute force search              & N/A   \\
1NN KD Tree                         & N/A   \\
1NN Ball Tree                       & N/A   \\ \midrule \midrule
1NN K-means 128 Clusters            & 0.96      \\
1NN K-means 256 Clusters            & 0.97      \\
1NN K-means 512 Clusters            & 0.99      \\ \midrule
1NN QK-means 128 Clusters           & 0.74     \\
1NN QK-means 256 Clusters           & 0.66      \\
1NN QK-means 512 Clusters           & 0.66      \\ \midrule \midrule
Nyström K-means + SVM 128 Clusters  & 0.98      \\
Nyström K-means + SVM 256 Clusters  & 1.0      \\
Nyström K-means + SVM 512 Clusters  & 1.0      \\ \midrule
Nyström QK-means + SVM 128 Clusters & 0.95      \\
Nyström QK-means + SVM 256 Clusters & 1.0      \\
Nyström QK-means + SVM 512 Clusters & 1.0      \\ \bottomrule
\end{tabular}

\caption{Results on the classification task on \texttt{Blobs} dataset. Results are averaged over 5 runs. ``N/A'' denotes experiments that did not finish. Only results with sparsity value 5 are displayed for \qkmeans experiments. For the \qkmeans results, only those obtained with sparsity level = 5 are displayed.}
\label{table:results_blobs}

\end{table}

\begin{table}[]

\centering
\begin{tabular}{@{}l|P{2.5cm}|P{2.5cm}}
\toprule
                                    & Accuracy \texttt{Fashion-MNIST} & Accuracy \texttt{MNIST} \\ \midrule
1NN Brute force search              &  0.85   &  0.97 \\
1NN KD Tree                         & 0.85  &  0.97  \\
1NN Ball Tree                       & 0.85  &  0.97  \\ \midrule \midrule
1NN K-means 10 Clusters            & 0.84  &  0.96  \\
1NN K-means 16 Clusters            & 0.84  &  0.96  \\
1NN K-means 30 Clusters            & 0.84  &  0.96  \\ \midrule
1NN QK-means 10 Clusters           & 0.84  &  0.96  \\
1NN QK-means 16 Clusters           & 0.84  &  0.96  \\
1NN QK-means 30 Clusters           & 0.84  &  0.96  \\ \midrule \midrule
Nyström K-means + SVM 10 Clusters  &  0.71  & 0.74  \\
Nyström K-means + SVM 16 Clusters  &  0.75  & 0.83  \\
Nyström K-means + SVM 30 Clusters  &  0.78  & 0.88  \\ \midrule
Nyström QK-means + SVM 10 Clusters &  0.71  & 0.74  \\
Nyström QK-means + SVM 16 Clusters &  0.74  & 0.82  \\
Nyström QK-means + SVM 30 Clusters &  0.77  & 0.88  \\ \bottomrule
\end{tabular}

\caption{Results on the classification task on the \texttt{MNIST} and \texttt{Fashion-MNIST} datasets. Results are averaged over 5 runs. ``N/A'' denotes experiments that did not finish. For the \qkmeans results, only those obtained with sparsity level = 5 are displayed.}
\label{table:results_mnist_fmnist}

\end{table}

\subsection{Nyström approximation}

In this sub-section, we show how we can take advantage of the fast-operator obtained as output of our \qkmeans algorithm in order to speed-up the computation in the Nyström approximation. 
We start by giving background knowledge on the Nyström approximation then we present some recent work aiming at accelerating it using well know fast-transform method. 
We finally stem on this work to present a novel approach based on our \qkmeans algorithm.

\subsubsection{Background on the Nyström approximation}

Standard kernel machines are often impossible to use in large-scale applications because of their high computational cost associated with the kernel matrix $\rmK$ which has $O(n^2)$ storage and $O(n^2d)$ computational complexity: $\forall i,j \in\intint{\nexamples}, \rmK_{i,j} = k(\rvx_i, \rvx_j)$. A well-known strategy to overcome this problem is to use the Nyström method which computes a low-rank approximation of the kernel matrix on the basis of some pre-selected landmark points. 

Given $K \ll n$ landmark points $\{\rmU_i\}_{i=1}^{K}$, the Nyström method gives the following approximation of the full kernel matrix:
\begin{equation}
 \label{eq:nystrom}
 \rmK \approx \tilde\rmK = \rmC\rmW^\dagger\rmC^T,
\end{equation}
with $\rmW \in \R^{K \times K}$ containing all the kernel values between landmarks: $\forall i,j \in [\![K]\!]~ \rmW_{i,j} = k(\rmU_i, \rmU_j)$; $\rmW^\dagger$ being the pseudo-inverse of $\rmW$ and $\rmC \in \R^{n \times K}$ containing the kernel values between landmark points and all data points: $\forall i \in [\![n]\!], \forall j \in [\![K]\!]~ \rmC_{i, j} = k(\rmX_i, \rmU_j)$.

\subsubsection{Efficient Nyström approximation}

A substantial amount of research has been conducted toward landmark point selection methods for improved approximation accuracy \cite{kumar2012sampling} \cite{musco2017recursive}, but much less has been done to improve computation speed. In \cite{si2016computationally}, the authors propose an algorithm to learn the matrix of landmark points with some structure constraint, so that its utilisation is fast, taking advantage of fast-transforms. This results in an efficient Nyström approximation that is faster to use both in the training and testing phases of some ulterior machine learning application.

Remarking that the main computation cost of the Nyström approximation comes from the computation of the kernel function between the train/test samples and the landmark points, \cite{si2016computationally} aim at accelerating this step. In particular, they focus on a family of kernel functions that has the following form:
\begin{equation}
 k(\rvx_i, \rvx_j) = f(\rvx_i) f(\rvx_j) g(\rvx_i^T\rvx_j),
\end{equation}
where $f: \R^d \rightarrow \R$ and $g: \R \rightarrow \R$. They show that this family of functions contains some widely used kernels such as the Gaussian and the polynomial kernel. Given a set of $K$ landmark points $\rmU \in \R^{K \times d}$ and a sample $\rvx$, the computational time for computing the kernel between $\rvx$ and each row of $\rmU$ (necessary for the Nyström approximation) is bottlenecked by the computation of the product $\rmU\rvx$. They hence propose to write the $\rmU$ matrix as the concatenation of structured $s = K / d$ product of matrices:
\begin{equation}
 \rmU = \left[ \rmV_1 \rmH^T, \cdots, \rmV_s\rmH^T  \right]^T,
\end{equation}
where the $\rmH$ is a $d \times d$ matrix associated with a fast transform such as the \textit{Haar} or \textit{Hadamard} matrix, and the $\rmV_i$s are some $d \times d$ diagonal matrices to be either chosen with a standard landmark selection method or learned using an algorithm they provide.

Depending on the $\rmH$ matrix chosen, it is possible to improve the time complexity for the computation of $\rmU\rvx$ from $O(Kd)$ to $O(K \log{d})$ (\textit{Fast Hadamard transform}) or $O(K)$ (\textit{Fast Haar Transform}).

\subsubsection{\qkmeans in Nyström}

We propose to use our \qkmeans algorithm in order to learn directly the $\rmU$ matrix in the Nyström approximation so that the matrix-vector multiplication $\rmU \rvx$ is cheap to compute, but the structure of $\rmU$ is not constrained by some pre-defined transform matrix. We propose to take the objective $\rmU$ matrix as the \kmeans matrix of $\rmX$ since it has been shown to achieve good reconstruction accuracy in the Nyström method \cite{kumar2012sampling}.

As shown in the next sub-section, our algorithm allow to obtain an efficient Nyström approximation, while not reducing too much the quality of the \kmeans landmark points which are encoded as a factorization of sparse matrix. 

\subsubsection{Results}

The Figure~\ref{fig:nystrom} summarizes the results achieved in the Nyström approximation setting. 

The Figures on the right display the average time for computing one line of the approximated matrix in Equation~\ref{eq:nystrom}. In Figure~\ref{fig:blobs:nystrom_time}, we clearly see the speed-up offered using the \qkmeans method on the \texttt{Blobs} dataset. On the \texttt{Mnist} and \texttt{Fashion-MNIST} dataset (Figure~\ref{fig:mnist:nystrom_time}~and~\ref{fig:fashmnist:nystrom_time}), this speed-up is sensible but not as clear because the standard deviation is much larger. 

The Figures on the left show the approximation error of the Nyström approximation based on different sampling schemes w.r.t. the real kernel matrix. This error is computed by the Froebenius norm of the difference between the matrices and then normalized:

\begin{equation}
 error = \frac{||\rmK - \tilde\rmK||_F}{||\rmK||_F}
\end{equation}

. The \qkmeans approach gives better reconstruction error than the Nyström method based on uniform sampling although they are slightly worse than the one obtained with the \kmeans centroids. We see that that the difference in approximation error between \kmeans and \qkmeans is almost negligeable when compared to the approximation error obtained with the uniform sampling scheme.

From a more practical point of view, we show in Table~\ref{table:results_blobs} and Table~\ref{table:results_mnist_fmnist} that the Nyström approximation based on \qkmeans can then be used in a linear SVM and achieve as good performance as the one based on the \kmeans approach.

\begin{figure}
\begin{subfigure}[b]{.49\textwidth}
\includegraphics[width=\textwidth]{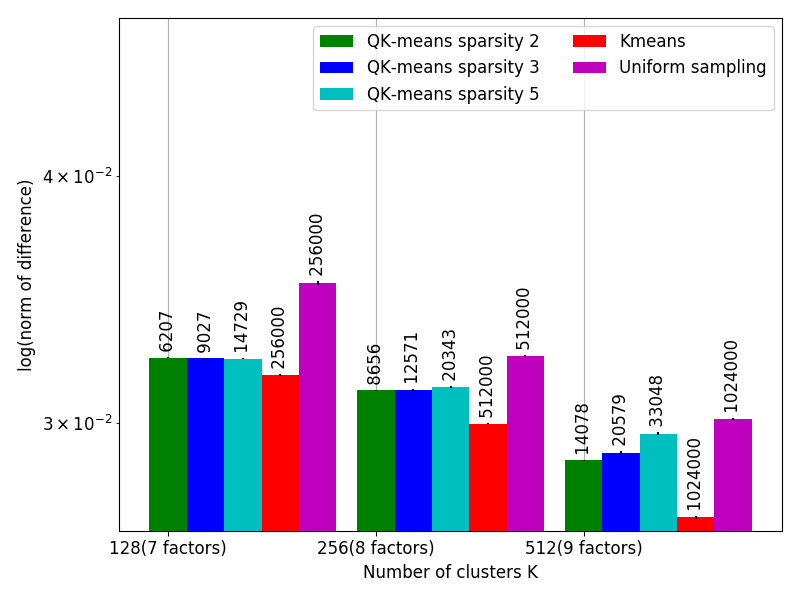}
\caption{\texttt{Blobs}: Nyström reconstruction error.}
\label{fig:blobs:nystrom_error}
\end{subfigure}
\begin{subfigure}[b]{.49\textwidth}
\includegraphics[width=\textwidth]{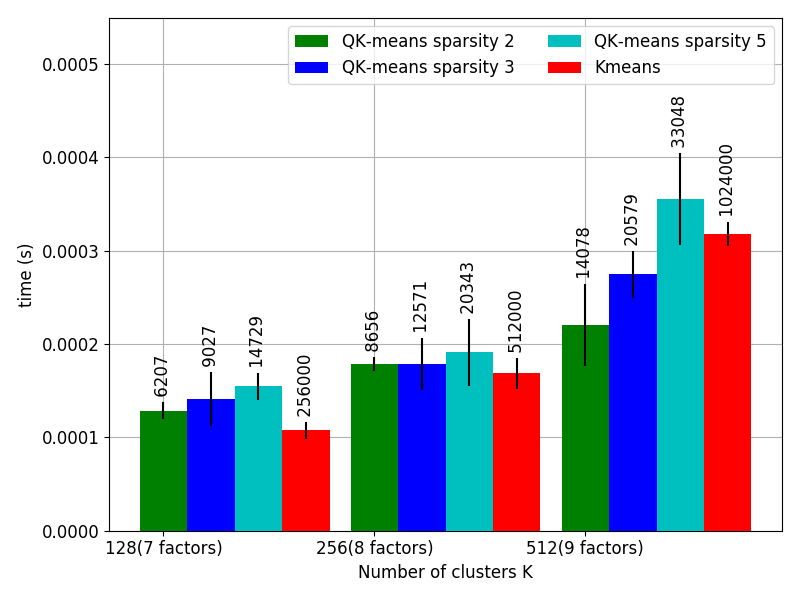}
\caption{\texttt{Blobs}: Nyström inference time.}
\label{fig:blobs:nystrom_time}
\end{subfigure}
\begin{subfigure}[b]{.49\textwidth}
\includegraphics[width=\textwidth]{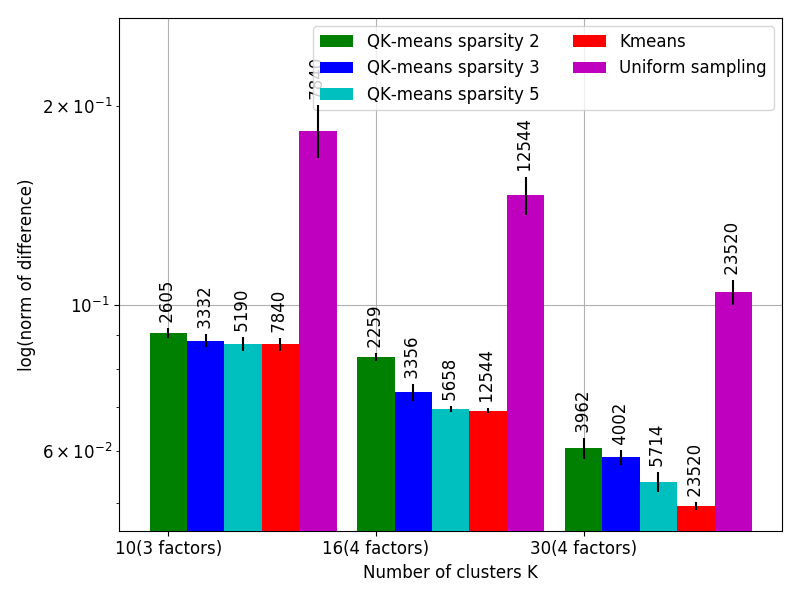}
\caption{\texttt{MNIST}: Nyström reconstruction error.}
\label{fig:mnist:nystrom_error}
\end{subfigure}
\begin{subfigure}[b]{.49\textwidth}
\includegraphics[width=\textwidth]{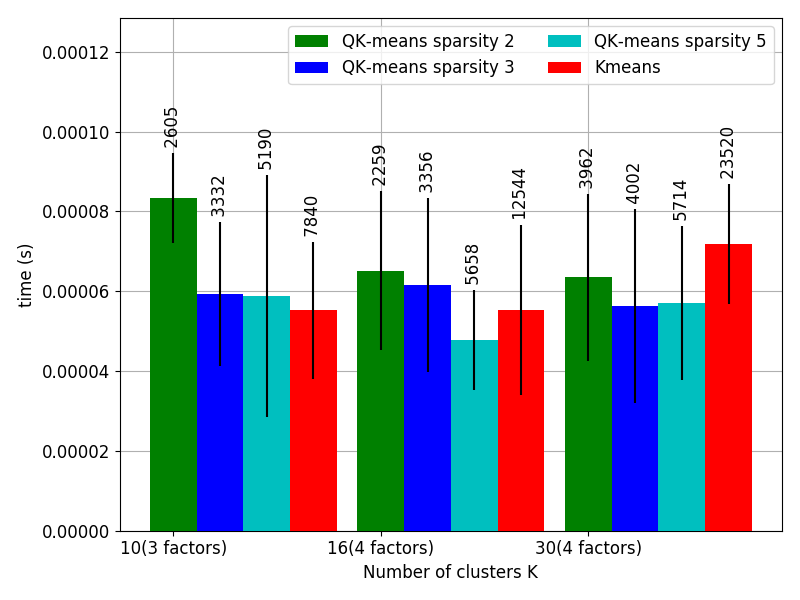}
\caption{\texttt{MNIST}: Nyström inference time.}
\label{fig:mnist:nystrom_time}
\end{subfigure}
\begin{subfigure}[b]{.49\textwidth}
\includegraphics[width=\textwidth]{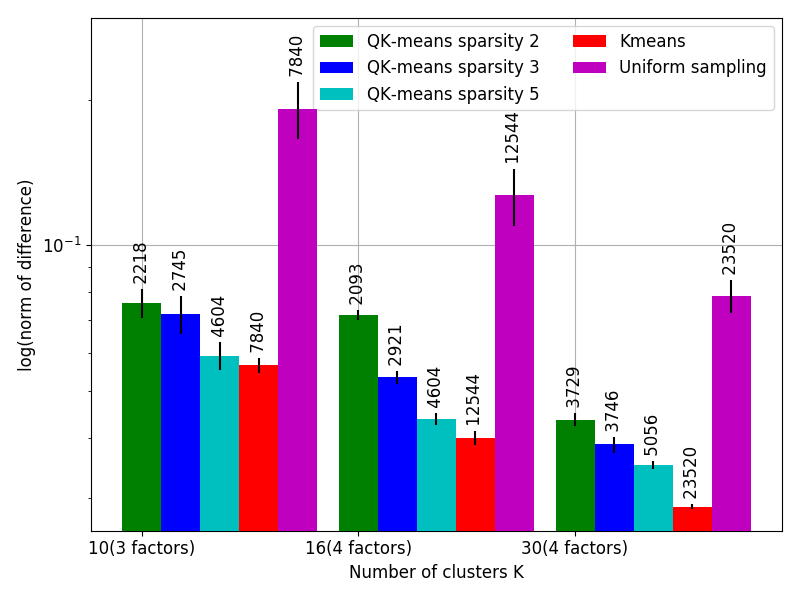}
\caption{\texttt{Fashion-MNIST}: Nyström reconstruction error.}
\label{fig:fashmnist:nystrom_error}
\end{subfigure}
\begin{subfigure}[b]{.49\textwidth}
\includegraphics[width=\textwidth]{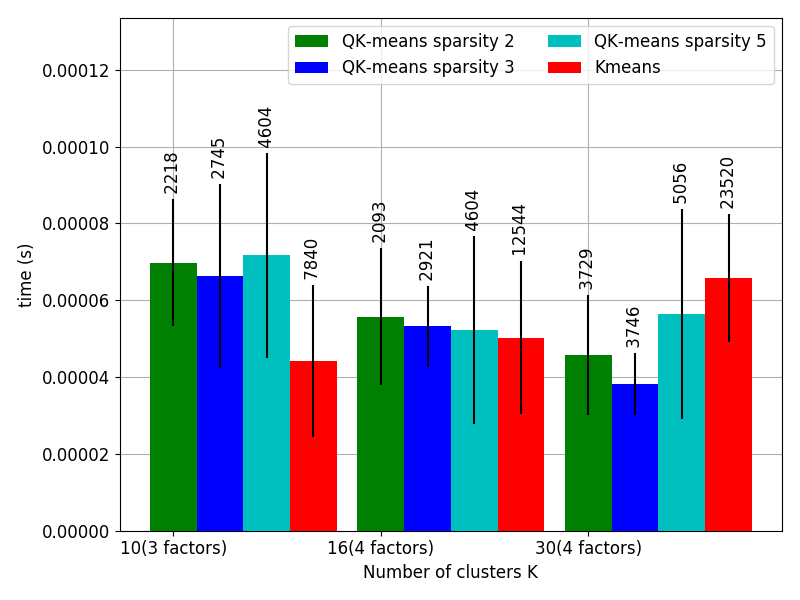}
\caption{\texttt{Fashion-MNIST}: Nyström inference time.}
\label{fig:fashmnist:nystrom_time}
\end{subfigure}
\caption{Nystr\"om approximation results: accuracy (left) and running times (right). The uniform sampling based Nyström approximation running times are not displayed because they are the same as for the Nyström approximation based on \kmeans centroids. Every experiment results are averaged over 5 runs. The vertical black lines are the standard deviation w.r.t. the runs.}
\label{fig:nystrom}
\end{figure}

\section{Conclusion}
\label{sec:conclusion}

In this paper, we have proposed a variant of the \kmeans algorithm, named \qkmeans, designed to achieve a similar goal -- clustering data points around $\nclusters$ learned centroids -- with a much lower computational complexity as the dimension of the data, the number of examples and the number of clusters get high. Our approach is based on the approximation of the centroid matrix by an operator structured as a product of a small number of sparse matrices, resulting in a low time and space complexity when applied to data vectors.
We have shown the convergence properties of the proposed algorithm and provided its complexity analysis.

An implementation prototype has been run in several core machine learning use cases including clustering, nearest-neighbor search and Nystr\"om approximation. The experimental results illustrate the computational gain in high dimension at inference time as well as the good approximation qualities of the proposed model.

Beyond these modeling, algorithmic and experimental contributions to low-complexity high-dimensional machine learning, we have identified several important questions that are still to be addressed.
First, although learning the fast-structure operator has been nicely integrated in the training algorithm with an advantageous theoretical time and space complexity, exhibiting gains in actual running times has not been achieved yet for the \qkmeans learning procedure, compared to \kmeans.
This may be obtained in even higher dimensions than in the proposed experimental settings, which may require a new version of \qkmeans using batches of data in order to process amounts of data that do not fit in memory.
Second, the expressiveness of the fast-structure model is still to be theoretically studied, while our experiments seems to show that arbitrary matrices may be well fitted by such models.
Third, we believe that learning fast-structure linear operators during the training procedure may be generalized to many core machine learning methods in order to speed them up and make them scale to larger dimensions.

%====================================================================================

\bibliographystyle{plain}
\bibliography{qmeans}

\begin{thebibliography}{10}

\bibitem{arthur2006slow}
David Arthur and Sergei Vassilvitskii.
\newblock How slow is the k-means method?
\newblock In {\em Symposium on Computational Geometry}, pages 1--10, 2006.

\bibitem{bolte2014proximal}
J{\'e}r{\^o}me Bolte, Shoham Sabach, and Marc Teboulle.
\newblock Proximal alternating linearized minimization or nonconvex and
  nonsmooth problems.
\newblock {\em Mathematical Programming}, 146(1-2):459--494, 2014.

\bibitem{boutsidis2014randomized}
Christos Boutsidis, Anastasios Zouzias, Michael~W Mahoney, and Petros Drineas.
\newblock Randomized dimensionality reduction for $ k $-means clustering.
\newblock {\em IEEE Transactions on Information Theory}, 61(2):1045--1062,
  2014.

\bibitem{gittens2016revisiting}
Alex Gittens and Michael~W Mahoney.
\newblock Revisiting the nystr{\"o}m method for improved large-scale machine
  learning.
\newblock {\em The Journal of Machine Learning Research}, 17(1):3977--4041,
  2016.

\bibitem{hartigan1979algorithm}
John~A Hartigan and Manchek~A Wong.
\newblock Algorithm as 136: A k-means clustering algorithm.
\newblock {\em Journal of the Royal Statistical Society. Series C (Applied
  Statistics)}, 28(1):100--108, 1979.

\bibitem{jain2010data}
Anil~K Jain.
\newblock Data clustering: 50 years beyond k-means.
\newblock {\em Pattern recognition letters}, 31(8):651--666, 2010.

\bibitem{kumar2012sampling}
Sanjiv Kumar, Mehryar Mohri, and Ameet Talwalkar.
\newblock Sampling methods for the nystr{\"o}m method.
\newblock {\em Journal of Machine Learning Research}, 13(Apr):981--1006, 2012.

\bibitem{le2013fastfood}
Quoc Le, Tam{\'a}s Sarl{\'o}s, and Alex Smola.
\newblock Fastfood—approximating kernel expansions in loglinear time.
\newblock In {\em International Conference on Machine Learning}, 2013.

\bibitem{LeMagoarou2016Flexible}
Luc Le~Magoarou and R\'emi Gribonval.
\newblock Flexible multilayer sparse approximations of matrices and
  applications.
\newblock {\em {{IEEE} Journal of Selected Topics in Signal Processing}},
  10(4):688--700, 2016.

\bibitem{lecun-mnisthandwrittendigit-2010}
Yann LeCun and Corinna Cortes.
\newblock {MNIST} handwritten digit database, 2010.

\bibitem{liu2017sparse}
Weiwei Liu, Xiaobo Shen, and Ivor Tsang.
\newblock Sparse embedded $ k $-means clustering.
\newblock In {\em Advances in Neural Information Processing Systems}, pages
  3319--3327, 2017.

\bibitem{Morgenstern1975Linear}
Jacques Morgenstern.
\newblock {The Linear Complexity of Computation}.
\newblock {\em {Journal of the ACM}}, 22(2):184--194, April 1975.

\bibitem{muja2014scalable}
Marius Muja and David~G Lowe.
\newblock Scalable nearest neighbor algorithms for high dimensional data.
\newblock {\em IEEE transactions on pattern analysis and machine intelligence},
  36(11):2227--2240, 2014.

\bibitem{musco2017recursive}
Cameron Musco and Christopher Musco.
\newblock Recursive sampling for the nystrom method.
\newblock In {\em Advances in Neural Information Processing Systems}, pages
  3833--3845, 2017.

\bibitem{Pedregosa2011Scikit}
F.~Pedregosa, G.~Varoquaux, A.~Gramfort, V.~Michel, B.~Thirion, O.~Grisel,
  M.~Blondel, P.~Prettenhofer, R.~Weiss, V.~Dubourg, J.~Vanderplas, A.~Passos,
  D.~Cournapeau, M.~Brucher, M.~Perrot, and E.~Duchesnay.
\newblock Scikit-learn: Machine learning in {P}ython.
\newblock {\em Journal of Machine Learning Research}, 12:2825--2830, 2011.

\bibitem{que2016back}
Qichao Que and Mikhail Belkin.
\newblock Back to the future: Radial basis function networks revisited.
\newblock In {\em AISTATS}, pages 1375--1383, 2016.

\bibitem{Sculley2010Web}
David Sculley.
\newblock Web-scale k-means clustering.
\newblock In {\em Proceedings of the 19th international conference on World
  wide web}, pages 1177--1178. ACM, 2010.

\bibitem{shen2017compressed}
Xiaobo Shen, Weiwei Liu, Ivor Tsang, Fumin Shen, and Quan-Sen Sun.
\newblock Compressed k-means for large-scale clustering.
\newblock In {\em Thirty-First AAAI Conference on Artificial Intelligence},
  2017.

\bibitem{si2016computationally}
Si~Si, Cho-Jui Hsieh, and Inderjit Dhillon.
\newblock Computationally efficient nystr{\"o}m approximation using fast
  transforms.
\newblock In {\em International Conference on Machine Learning}, pages
  2655--2663, 2016.

\bibitem{van2016local}
Twan Van~Laarhoven and Elena Marchiori.
\newblock Local network community detection with continuous optimization of
  conductance and weighted kernel k-means.
\newblock {\em The Journal of Machine Learning Research}, 17(1):5148--5175,
  2016.

\end{thebibliography}

\appendix
\section{\palm algorithm}
\label{sec:app:palm4msa}

The \palm algorithm~\cite{LeMagoarou2016Flexible} is given in Algorithm~\ref{algo:palm4msa} together with the time complexity of each line, using $A=\min(\nclusters, \datadim)$ and $B=\max(\nclusters, \datadim)$. 
Even more general constraints can be used, the constraint sets $\mathcal{E}_q$ are typically defined as the intersection of the set of unit Frobenius-norm matrices and of a set of sparse matrices.
The unit Frobenius norm is used together with the $\lambda$ factor to avoid a scaling indeterminacy.
Note that to simplify the model presentation, factor $\lambda$ is used internally in \palm and is integrated in factor $\rmS_1$ at the end of the algorithm (Line~\ref{line:palm:postprocess:S1}) so that $\rmS_1$ does not satisfy the unit Frobenius norm in $\mathcal{E}_1$ at the end of the algorithm.
The sparse constraints we used, as in~\cite{LeMagoarou2016Flexible}, consist of trying to have a given number of non-zero coefficients in each row and in each column.
This number of non-zero coefficients is called sparsity level in this paper.
In practice, the projection function at Line~\ref{line:palm:update:S} keeps the largest non-zero coefficients in each row and in each column, which only guarantees
the actual number of non-zero coefficients is at least equal to the sparsity level.

\begin{algorithm}
	\caption{\palm algorithm}
	\label{algo:palm4msa}
	\begin{algorithmic}[1]
		
		\REQUIRE The matrix to factorize $\rmU \in \R^{\nclusters \times \datadim}$, the desired number of factors $\nfactors$, the constraint sets $\mathcal{E}_q$ , $q\in \intint{\nfactors}$ and a stopping criterion (e.g., here, a number of iterations $I$ ).
		
		\STATE $\lambda \leftarrow \norm{S_1}_F$
		\COMMENT{$\bigO{B}$}
		\label{line:palm:init:lambda}
		\STATE $S_1 \leftarrow \frac{1}{\lambda} S_1$
		\COMMENT{$\bigO{B}$}
		\label{line:palm:normalize:S1}
		\FOR {$i \in\intint{I}$ while the stopping criterion is not met}
		\FOR {$q = \nfactors$ down to $1$}
		\STATE  $\rmL_q \leftarrow \prod_{l=1}^{q-1} \rmS_{l}^{(i)}$
		\label{line:palm:L}
		\STATE  $\rmR_q \leftarrow \prod_{l=q+1}^{\nfactors} \rmS_{l}^{(i+1)}$
		\label{line:palm:R}
		\STATE Choose $c > \lambda^2 ||\rmR_q||_2^2 ||\rmL_q||_2^2$
		\COMMENT{$\bigO{A \log A+B}$}
		\label{line:palm:c}
		\STATE $\rmD \leftarrow \rmS_q^i - \frac{1}{c} \lambda \rmL_q^T\left (\lambda\rmL_q \rmS_q^i \rmR_q - \rmU\right )\rmR_q^T$
		\COMMENT{$\bigO{AB\log A}$}
		\label{line:palm:D}
		\STATE $\rmS^{(i+1)}_q \leftarrow P_{\mathcal{E}_q}(\rmD)$
		\COMMENT{$\bigO{A^2\log A}$ or $\bigO{AB\log B}$}
		\label{line:palm:update:S}
		\ENDFOR
		\STATE $\hat \rmU \eqdef \prod_{j=1}^{\nfactors} \rmS_q^{(i+1)}$
		\COMMENT{$\bigO{A^2\log A + AB}$}
		\label{line:palm:U}
		\STATE $\lambda \leftarrow \frac{Trace(\rmU^T\hat\rmU)}{Trace(\hat\rmU^T\hat\rmU)}$
		\COMMENT{$\bigO{AB}$}
		\label{line:palm:update:lambda}
		\ENDFOR
		\STATE $S_1 \leftarrow \lambda S_1$
		\COMMENT{$\bigO{B}$}
		\label{line:palm:postprocess:S1}
		\ENSURE $\left \lbrace \rmS_q : \rmS_q \in \mathcal{E}_q\right \rbrace_{q\in\intint{\nfactors}}$ such that $\prod_{q\in\intint{\nfactors}}\rmS_q \approx \rmU$
		
	\end{algorithmic}
\end{algorithm}

The complexity analysis is proposed under the following assumptions, which are satisfied in the mentioned applications and experiments: the number of factors is $\nfactors=\mathcal{O}\left (\log A\right )$; all but one sparse factors are of shape $A \times A$ and have $\bigO{A}$ non-zero entries while one of them is of shape $A\times B$ or $B\times A$ with $\bigO{B}$ non-zero entries.
In such conditions, the complexity of each line is:
\begin{itemize}
 \item [Lines~\ref{line:palm:init:lambda}-\ref{line:palm:normalize:S1}] Computing these normalization steps is linear in the number of non-zeros coefficients in $\rmS_1$.
 \item [Lines~\ref{line:palm:L}-\ref{line:palm:R}] Fast operators $\rmL$ and $\rmR$ are defined for subsequent use without computing explicitly the product.
 \item [Line~\ref{line:palm:c}] The spectral norm of $\rmL$ and $\rmR$ is obtained via a power method by iteratively applying each operator, benefiting from the fast transform.
 \item [Line~\ref{line:palm:D}] The cost of the gradient step is dominated by the product of sparse matrices.
\item [Line~\ref{line:palm:update:S}] The projection onto a sparse-constraint set takes $\bigO{A^2\log A}$ for all the $A\times A$ matrices and $\bigO{AB\log B}$ for the rectangular matrix at the leftmost or the rightmost position.
 \item [Line~\ref{line:palm:U}] The reconstructed matrix $\hat \rmU$ is computed using $\bigO{\log A}$ products between $A\times A$ sparse matrices, in $\bigO{A^2}$ operations each, and one product with a sparse matrix in $\bigO{AB}$.
 \item [Line~\ref{line:palm:update:lambda}] The numerator and denominator can be computed using a Hadamard product between the matrices followed by a sum over all the entries.
  \item [Line~\ref{line:palm:postprocess:S1}] Computing renormalization step is linear in the number of non-zeros coefficients in $\rmS_1$.
\end{itemize}

Hence, the overal time complexity of \palm is in $\bigO{AB\left (\log^2 A+\log B\right )}$, due to Lines~\ref{line:palm:D} and~\ref{line:palm:update:S}.

\end{document}